\documentclass{article}

\usepackage{microtype}
\usepackage{graphicx}
\usepackage{subcaption}
\usepackage{booktabs} % for professional tables

\usepackage{hyperref}
\usepackage{algorithmic}
\usepackage{nicefrac}
\usepackage[preprint]{icml2026}
\usepackage{amsmath}
\usepackage{amssymb}
\usepackage{mathtools}
\usepackage{amsthm}

\usepackage[capitalize,noabbrev]{cleveref}

\theoremstyle{plain}
\newtheorem{theorem}{Theorem}[section]

\theoremstyle{definition}
\newtheorem{definition}[theorem]{Definition}

\theoremstyle{remark}

\usepackage[textsize=tiny]{todonotes}

\icmltitlerunning{Towards Robust Scaling Laws for Optimizers}

\begin{document}

\twocolumn[
  \icmltitle{Towards Robust Scaling Laws for Optimizers}



  \icmlsetsymbol{equal}{*}

  \begin{icmlauthorlist}
    \icmlauthor{Alexandra Volkova}{ista}
    \icmlauthor{Mher Safaryan}{lancaster}
    \icmlauthor{Christoph H. Lampert}{ista}
    \icmlauthor{Dan Alistarh}{ista,redhat}
  \end{icmlauthorlist}

  \icmlaffiliation{ista}{Institute of Science and Technology Austria (ISTA), Klosterneuburg, Austria}
  \icmlaffiliation{redhat}{Red Hat, Raleigh, NC, USA}
  \icmlaffiliation{lancaster}{Lancaster University, Lancaster, UK}

  \icmlcorrespondingauthor{Alexandra Volkova}{avolkova@ist.ac.at}

  \icmlkeywords{Machine Learning, ICML}

  \vskip 0.3in
]

\renewcommand{\paragraph}[1]{\vspace{0.1em}\noindent\textbf{#1}}



\printAffiliationsAndNotice{}  % no special notice (required even if empty)

\begin{abstract}

The quality of Large Language Model (LLM) pretraining depends on multiple factors, including the compute budget and the choice of optimization algorithm. Empirical scaling laws are widely used to predict loss as model size and training data grow, however, almost all existing studies fix the optimizer (typically AdamW). At the same time, a new generation of optimizers (e.g., Muon, Shampoo, SOAP) promises faster and more stable convergence, but their relationship with model and data scaling is not yet well understood. In this work, we study scaling laws across different optimizers. Empirically, we show that 1) separate Chinchilla-style scaling laws for each optimizer are ill-conditioned and have highly correlated parameters. Instead, 2) we propose a more robust law with shared power-law exponents and optimizer-specific rescaling factors, which enable direct comparison between optimizers. Finally, 3) we provide a theoretical analysis of gradient-based methods for the proxy task of a convex quadratic objective, demonstrating that Chinchilla-style scaling laws emerge naturally as a result of loss decomposition into irreducible, approximation, and optimization errors. 
\end{abstract}

\vspace{-1em}
\section{Introduction}

Pretraining Large Language Models (LLMs) requires extreme computational resources, and current trends point toward even larger models and longer training runs. In this light, scaling laws, the empirical  relations between loss, model size, and training data size, have become a key practical tool for planning training runs and allocating resources \citep{kaplan2020scalinglaws,hoffmann_chinchilla}, and recent work has provided a better understanding of these relationships from the theoretical side~\citep{kunstnerbach2025scaling, bach2025ztransform}. 

The most popular form is given by the ``Chinchilla'' scaling laws \cite{hoffmann_chinchilla} representing the loss of a trained model as a decoupled function of number of model parameters $N$ and training items (tokens) $D$:
\begin{equation}
    L = \frac{A}{N^\alpha} + \frac{B}{D^\beta}  + E,
    \label{eq:classic-law}
\end{equation}
where $A$ and $B$ set the term scales, $\alpha, \beta$ are power-law exponents, and $E$ is an irreducible error level. Here, $L$ is a cross-entropy loss in the next token prediction task. This form recognizes two independent sources of loss: finite model capacity $N^{-\alpha}$ and finite data scope $D^{-\beta}$. 

A parallel trend is that, for the last decade, AdamW~\citep{kingma2015adam, Loshchilov2017DecoupledWD} has become the de facto standard choice for training machine learning models, in particular LLMs, because of its tunability, robustness, and solid performance across architectures and data regimes. It is natural that the majority of scaling law studies treat the optimizer as a fixed option, and not as an optimization parameter. Yet, this approach may be insufficient given the emergence of several promising alternative optimizers, some of which are shown to be more effective than AdamW for pretraining~\cite{liu2025muonscalablellmtraining, vlassis2025beyond}.

% motivated by geometric structure of model layers or higher-order optimization. 
Illustrative examples are given by novel optimizers such as i) \textit{Muon}~\cite{muon}, a geometry-aware optimizer that orthogonalizes weight updates, ii) \textit{Scion}~\cite{scion}, an optimizer that extends this approach by explicitly constraining the update norm, iii) \textit{Shampoo}~\citep{shampoo}, a second-order method that preconditions gradients to approximate curvature information, and iv) \textit{SOAP}~\citep{soap}, a further refinement of Shampoo that runs Adam in the eigenbasis of Shampoo’s preconditioner. Each of those optimizers aims to improve upon AdamW, either utilizing geometric structure of model layers or approximating higher-order derivatives. Despite significant interest in this area, the community currently lacks a clear framework of comparing the performance of these optimizers outside of individual empirical studies~\citep{semenov2025benchmarking, wen2025fantasticpretrainingoptimizers}, leading to confusing and even sometimes contradictory performance claims.

\paragraph{Contribution.} In this work, we develop a systematic recipe for comparing optimizers within the scaling law framework, with the goal of understanding how optimizer choice interacts with loss scaling in both model size and data, whether this effect can be captured by a unified law, and what theoretical and practical conclusions follow.

A natural approach, used in prior work \cite{wen2025fantasticpretrainingoptimizers}, is to fit standard scaling laws separately for each optimizer and extrapolate from small-scale runs. We show that this approach is very brittle: the inferred scaling parameters are often poorly identified and can vary substantially under resampling and across runs, even when the predicted loss curves are similar. This motivates the need for stable optimizer-aware scaling laws.

We resolve this by providing a first robust law for optimizer scaling. The key insight is that we propose to fit a set of \emph{shared exponents} on a reference optimizer (usually AdamW), and then re-fit optimizer-specific scaling factors that are both stable and interpretable. On the empirical side, we show that our approach significantly improves parameter variance, and drastically reduces extrapolation error relative to standard fitting.  
On the theoretical side, we build on recent analyses of gradient-based methods on convex quadratics with power-law spectra~\citep{bach2025ztransform}, and show that a Chinchilla-style loss decomposition arises by separating irreducible error, approximation error (finite model width), and optimization error (finite steps).

In summary, our contributions are as follows:
\begin{itemize}
    \item We demonstrate the brittleness of standard independent fits of the classic Chinchilla form when applied to optimizers: the fits are numerically unstable and yield highly-correlated parameters. 
    \item We propose a new unified scaling law with shared parameters across optimizers, but with optimizer-specific rescaling factors that render the approach both more stable and more interpretable, allowing for direct optimizer comparison. 
     \item Theoretical justification: We provide an analytical derivation for the Chinchilla-form law in the case of a classic quadratic model, and prove that the scaling exponents are directly linked to the spectral properties of the problem. 
     
    \item We show empirically that our approach leads to significantly tighter and more stable fits across several popular optimizers (AdamW, Shampoo, Muon, Scion, and SOAP) and across two different architectures and datasets. We also show that the law can be specialized for compute allocation, where we show that the choice of optimizer can be modeled as rescaling the data term $D$ and architecture choice can be viewed as rescaling the number of parameters $N$.    
   
\end{itemize}

\section{Background and Related Work}

\paragraph{Optimizers.} 
Adam~\citep{kingma2015adam} and its variant AdamW~\citep{Loshchilov2017DecoupledWD} have become the standard optimizers in settings such as LLM training. 

Among a new generation of optimizers, Shampoo~\cite{shampoo} and SOAP~\cite{soap} represent quasi-second-order optimization methods that approximate full-matrix preconditioning. Shampoo maintains left and right second-moment gradient statistics per 2D weight matrix, which are then used for the weight update via inverse matrix-root preconditioning.    SOAP  builds on the observation that Shampoo is equivalent to running an Adam-like method in the eigenbasis of Shampoo’s preconditioner. It runs Adam-style moment update in that slowly changing basis, improving stability and convergence speed. 

Muon~\cite{muon} and Scion~\cite{scion} use \textit{orthogonalized} updates aimed at layers containing matrix weights. Each 2D gradient matrix is first used to produce a standard SGD-momentum update. This update is then replaced with the nearest semi-orthogonal matrix, which is efficiently computed via several Newton--Schultz iterations.  Muon uses this matrix for its update, while Scion performs a constrained optimization step via a Frank-Wolfe iteration.

\paragraph{Optimizer Benchmarking.} Earlier benchmarking work by \citet{schmidt-bench} systematically compared optimizers across various tasks and architectures, though not specifically focused on LLM pretraining at scale
More recently, \citet{semenov2025benchmarking} benchmarked 11 optimizers for LLM pretraining in a fixed data setup, across multiple model sizes, token budgets, and batch sizes.
The work by \citet{liu2025muonscalablellmtraining} on benchmarking optimizers  fitted scaling-law-style curves for Muon and AdamW and showed that, under compute-optimal training, Muon achieves up to $2\times$ better end-to-end efficiency (loss vs computation) than AdamW in full-precision LLM pretraining.  

\paragraph{Hyperparameter Scaling.} 
  More broadly, recent work has focused on \emph{hyperparameter transfer across scales}, aiming to predict near-optimal tuning as $N$ and $D$ grow~\cite{kadra2023scaling, muP}. For example,~\citet{wen2025fantasticpretrainingoptimizers} fit scaling laws w.r.t. learning rate, weight decay, etc. as functions of model size and data-to-model ratio for AdamW and a set of modern optimizers. \citet{chen2025HyperparameterTransfer} develop scaling rules for learning rate and weight decay for matrix-preconditioned optimizers and show that optimizer gains can disappear under incorrect hyperparameter scaling. \citet{scalingExponents} study transfer across parameterizations and optimizers, showing that multiple parameterizations can support hyperparameter transfer. 

These hyperparameter-scaling works are complementary to ours. They primarily address the question of how to tune a given optimizer as scale changes, whereas we focus on how optimizer choice shifts the loss–scale relationship itself once hyperparameters are appropriately tuned.

\citet{liu2025muonscalablellmtraining} compare Muon and AdamW optimizers via \emph{compute-only} scaling, fitting $L = A C^{-\beta}$ where $C$ denotes total training compute (FLOPs). This provides a concise summary of optimizer improvements at fixed compute, but it collapses the two fundamental scaling axes (model size $N$ and data $D$) into one. In contrast, our analysis keeps the $(N,D)$ structure and shows that optimizer effects are well-modeled as rescalings of effective $N$ and $D$, which yields a more interpretable separation into parameter- and data-efficiency factors.

Moreover,~\citet{wen2025fantasticpretrainingoptimizers} were interested in improved data efficiency of optimizers relative to AdamW, and fitted the law of the form:
\begin{equation}
    \begin{aligned}
        L_{opt} &= B D_{opt}^{-\beta} + E = B \left(\frac{D_{AdamW}}{s_{opt}}\right)^{-\beta}  + E, \\ & \text{ where }\; s_{opt} = \frac{D_{AdamW}}{D_{opt}}.
\end{aligned}
\label{eq:wen_eq}
\end{equation}
While useful, their law form does not take model scaling into account. In fact, they notice that data speedup slows down as model size grows, but it cannot be reflected in the law form of Eq.~\ref{eq:wen_eq}.

\section{Experimental Setup}

We begin by describing the experimental setup used to validate scaling laws. 
Specifically, we consider pretraining models from two different LLM families (Llama~\citep{touvron2023llama2openfoundation} and Olmo~\citep{olmo2}) across different datasets, spanning a range of model sizes and token-to-parameter ratios.

\paragraph{Models.} We train a set of two decoder-only LLMs from two different architecture families: OLMo-2~\cite{olmo2} and classic LLaMa-style models~\cite{touvron2023llama2openfoundation}. For both families, model parameters range from 50M to 1.5B, covering small and medium scale regimes. Model structure is detailed in Table~\ref{tab:architectures}. Model configuration details are provided in Appendix~\ref{apx:model-configs}. 

\begin{table}[h]
    \centering
        \caption{Architectural characteristics of model families used.}
    \label{tab:architectures}

    \begin{tabular}{lcc}
    \toprule
        & OLMo & LLaMa \\
        \midrule
        Norm & Post- & Pre- \\
        Attention & GQA, ratio$ \;=3$ & MHA \\
        Activations & ReLU$^2$ & SwiGLU \\
        QK norm & Yes & No \\
        Tied embeddings & Yes & Yes \\
        \bottomrule
    \end{tabular}
\end{table}

\paragraph{Datasets.} As training data, we use i) the ClimbMix dataset~\cite{climb_dataset} for the OLMo family models and ii) the FineWeb dataset~\cite{fineweb} for the LLaMa family. We do so to demonstrate independence of our results from the dataset and architecture used. For each model size $N$ we train for a set of token-to-parameter ratios $D/N = \{30, 50, 100, 200\}$. 
As such, we go significantly beyond the standard ``Chinchilla-optimal'' regime ($D/N = 20$), as we are interested in long-tail behavior.

\paragraph{Training setup.} We evaluate five optimizers: AdamW~\cite{kingma2015adam, Loshchilov2017DecoupledWD}, Muon~\cite{muon}, Scion~\cite{scion}, Shampoo~\cite{shampoo}, and SOAP~\cite{soap}.  We use a batch size of 512 with a context length of 1024 tokens. Our training uses Warmup-Stable-Decay (WSD)~\cite{hu2024minicpm}, allocating 5\% of total epochs to the warmup phase and 20\% to the decay phase. We sweep learning rate and all other optimizer hyperparameters ($\beta, \varepsilon,$ etc.) using the 50M model, picking the ones with the best loss on the \textit{validation} set. Later, we reuse them assuming that these parameters transfer across model sizes, and we additionally sweep over learning rates for each model size. We perform all sweeps for each architecture independently.

The experiments were run on 8$\times$H100 and 8$\times$B200 GPU nodes. Overall, we conducted over 250 training runs.

\paragraph{Scaling laws fitting procedure.} Each training run produces a data point $(N, D, L)$ for the scaling law fit, where $L$ is a loss on the \textit{test} set. Here $N$ is the total parameter count including embeddings, and $D$ is the number of training tokens. To estimate the hyperparameters of the law, we use least squares optimization from \texttt{scipy.optimize}~\cite{2020SciPy-NMeth} with the robust Huber loss ($\delta = 10^{-3}$). For numerical stability, we fit data points in the log-scale, $(\log N, \log D, \log L)$.

\begin{table*}[t]
    \centering
    \caption{Scaling law coefficients across optimizers. Coefficient errors, train and test errors are computed via leave-one-out cross-validation.}
    \label{tab:independent-sc-law}

    \begin{tabular}{c c c c c c c c}
    \toprule
    Optimizer & A & $\alpha$ & B & $\beta$ & E & Train fit error & Test fit error\\
    \midrule
AdamW & $4966_{\pm 4612}$ &$0.49_{\pm 0.03}$ &$1084_{\pm 453}$ &$0.38_{\pm 0.04}$ &$2.11_{\pm 0.03}$ &$1.3 \cdot 10^{-5}$ &$1.53 \cdot 10^{-5}$\\
Muon & $1148_{\pm 133}$ &$0.40_{\pm 0.01}$ &$110_{\pm 80}$ &$0.27_{\pm 0.03}$ &$1.90_{\pm 0.01}$ &$1.14 \cdot 10^{-5}$ &$1.16 \cdot 10^{-5}$\\
Scion & $449_{\pm 541}$ &$0.34_{\pm 0.03}$ &$58_{\pm 28}$ &$0.24_{\pm 0.05}$ &$1.80_{\pm 0.11}$ &$1.22 \cdot 10^{-5}$ &$1.36 \cdot 10^{-5}$\\
Shampoo & $3580_{\pm 4024}$ &$0.46_{\pm 0.05}$ &$96_{\pm 973}$ &$0.27_{\pm 0.08}$ &$2.00_{\pm 0.07}$ &$1.38 \cdot 10^{-5}$ &$1.95 \cdot 10^{-5}$\\
SOAP & $2175_{\pm 4723}$ &$0.44_{\pm 0.05}$ &$75_{\pm 83}$ &$0.26_{\pm 0.08}$ &$2.00_{\pm 0.23}$ &$1.29 \cdot 10^{-5}$ &$1.67 \cdot 10^{-5}$\\
\bottomrule
\end{tabular}
\end{table*}

\begin{figure*}[t]
  \centering
  \begin{subfigure}[b]{0.49\textwidth}
    \centering
    \includegraphics[width=\textwidth]{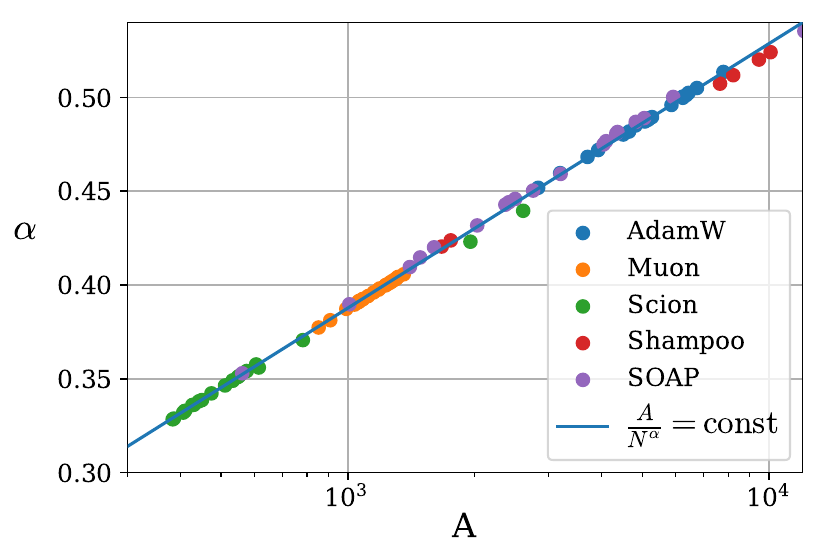}
    % \caption{Caption 1}
    \label{fig:error-correlation-A}
  \end{subfigure}
  \hfill
  \begin{subfigure}[b]{0.49\textwidth}
    \centering
    \includegraphics[width=\textwidth]{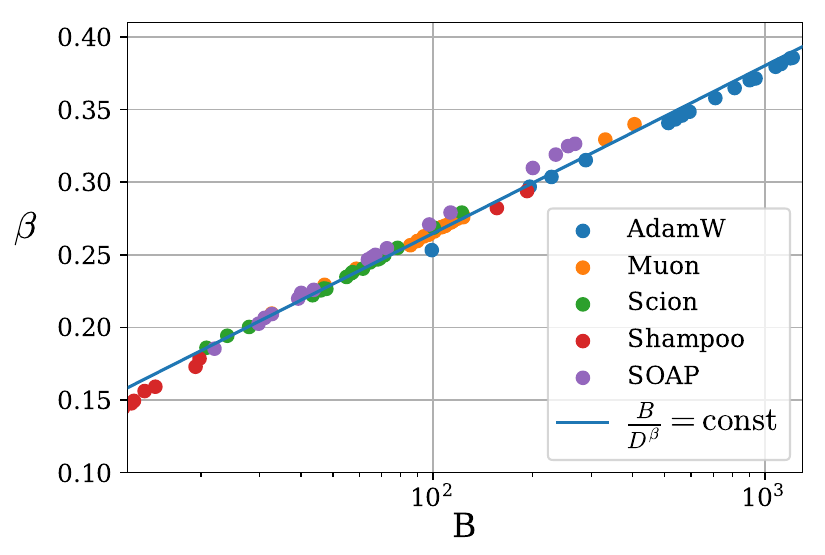}
    % \caption{Caption 2}
    \label{fig:error-correlation-B}
  \end{subfigure}
  \vspace{-1em}
  \caption{Correlation between estimated hyperparameters $A, \alpha$ and $B, \beta$ for leave-one-out cross-validation.}
  \label{fig:error-correlation}
\end{figure*}

\section{The Direct Approach: Fitting Independent Per-Optimizer Scaling Laws}
\label{sec:naive-fitting}

A natural approach to optimizer parametrization in the context of scaling laws~\citep{wen2025fantasticpretrainingoptimizers},  is to fit a separate Chinchilla-style scaling law for each optimizer based on small-scale empirical data, determine the corresponding parameters, and then extrapolate performance. However, we show that this approach is not robust in practice, and can yield misleading results. Specifically, when fitting Eq.~\ref{eq:classic-law} independently per optimizer, the estimated pairs $(A, \alpha)$ and $(B, \beta)$ exhibit a strong correlation despite producing similar predicted curves.

\paragraph{Fitting Procedure.} 
The direct approach used in prior work, is to simply fit scaling laws of the form given in Eq.~\ref{eq:classic-law} across all model sizes and data budgets, independently for each optimizer, as described in Algorithm~\ref{alg:per_optim_fit}. An instance of the results of such a fit is given in Table~\ref{tab:independent-sc-law}. 

\begin{algorithm}[t]
\caption{Naive per-optimizer Scaling Law (independent Chinchilla fits)}
\label{alg:per_optim_fit}
\begin{algorithmic}[1]
\REQUIRE Runs $\mathcal{R}$ for optimizers $\mathcal{O}$; each $r\in\mathcal{R}$ has $(o,N,D,L)$ with $o\in\mathcal{O}$.\\
 Scaling law $\hat L(N,D;\theta)= A N^{-\alpha} + B D^{-\beta} + E$, with $\theta=(A,\alpha,B,\beta,E)$. \\ Nonlinear least-squares solver \textsc{Fit}$(\cdot)$ with bounds.

\ENSURE Fitted parameters $\{\hat\theta_o\}$.

\FOR{optimizer $o \in \mathcal{O}$}
  \STATE Prepare the data\\ $X \leftarrow \{(\log N(r), \log D(r)) :o(r) = o\}$,\\
         $y \leftarrow \{\log L(r) :  o(r) = o\}$
  \STATE Choose initialization $\theta_{init}$ and bounds \{$\theta_{b}$\}
  \STATE $\hat\theta_o \leftarrow \textsc{Fit}\!\left(
      \underset{\theta}{\min} \underset{i}{\sum} \big(y_i - L(X_i,\theta)\big)^2,
      \theta_{init}, \theta_{b}   \right)$
\ENDFOR
\STATE \textbf{return} $\{\hat\theta_o\}$ 
\end{algorithmic}
\end{algorithm}

\paragraph{Lack of Robustness.} 
Examining the results and the quality of fit, it is easy to note the high variability of $A$ and $B$ parameters. A simple way to assess  robustness is to perform \textit{leave-one-out cross-validation}: for each training configuration in turn, we remove it from the dataset and refit Eq.~\ref{eq:classic-law} on the remaining points, as described in Algorithm~\ref{alg:loo_cv}. This allows us to evaluate the prediction error of the refitted law on the excluded data point and yields standard deviations for each estimated parameter; error intervals represent $\pm$ one standard deviation. We report all values in Table~\ref{tab:independent-sc-law}.

The resulting error intervals are significant; however, we found the parameters $(A, \alpha)$ and $(B, \beta)$ to be highly correlated and, moreover, lying on curves $ \frac{A}{N^{\alpha}} = \text{const}$ and  $ \frac{B}{D^{\beta}} =  \text{const}$ (Figure~\ref{fig:error-correlation}). This indicates that $(A, \alpha)$ and $(B, \beta)$ cannot be treated as independent hyperparameters because of the law form and the small range of $N$ and $D$ values, the law parameters are implicitly confined by $ \frac{A}{N^{\alpha}} =  \text{const}$ and $ \frac{B}{D^{\beta}} =  \text{const}$. In other words, various parameter sets $(A, \alpha, B, \beta, E)$ explain the same training points almost equally well because the parameters trade off against each other. 

\paragraph{Discussion.} On a higher level, this observation means that scaling laws of the form Eq.~\ref{eq:classic-law} fitted independently to each optimizer are intrinsically unstable, as small changes in the data points can lead to large, compensating pair-wise shifts in $(A, \alpha)$ and $(B, \beta)$, even though the combined effect on the fitted curve remains similar.  As a result, comparing optimizers via these fitted parameters (for example, comparing exponents $\alpha, \beta$) is not robust: the results can change significantly under small perturbations of the data. In turn, this reduces the predictive power of the approach.

\begin{algorithm}[t]
\caption{Leave-one-out (LOO) cross-validation for Scaling Law fitting}
\label{alg:loo_cv}
\begin{algorithmic}[1]
\REQUIRE  Runs $\mathcal{R}$ with optimizers $\mathcal{O}$; each run $r$ has $(o,N,D,L)$, $o \in \mathcal{O}$. \\
Scaling law fitting algorithm \textsc{Alg}$(\cdot): \mathcal{R} \rightarrow \theta$.
\ENSURE LOO parameter errors $\{\Delta \theta_o\}$.

\FOR{each optimizer $o \in \mathcal{O}$}
    \STATE $\mathcal{R}_o \leftarrow \{r\in\mathcal{R} : o(r)=o\}$, \quad $m \leftarrow |\mathcal{R}_o|$
    \FOR{$i = 1 \; \dots \; m$}
        \STATE Train set: $\mathcal{R}_o^{(-i)} \leftarrow \mathcal{R}_o \setminus \{r_i\}$
        \STATE $\hat\theta^{(-i)} \leftarrow \textsc{Alg}\!\left(\mathcal{R}_o^{(-i)}\right)$
        % \STATE Compute error on test:\\ 
        % $e = (L_i - L(N_i, D_i; )$
        % \STATE $\hat L_i^{(-i)} \leftarrow \hat L\!\left(N(r_i),D(r_i);\hat\theta_o^{(-i)}\right)$
        % \STATE $e_i \leftarrow L(r_i) - \hat L_i^{(-i)}$
        \STATE Store $\hat\theta_o^{(-i)}$
    \ENDFOR
    % \STATE $\mathrm{MSE}_o \leftarrow \frac{1}{m}\sum_{i=1}^{m} e_i^2$
    \STATE Compute parameter std: \\
    $\hat\theta_o = \frac{1}{m} \underset{i}{\sum} \theta^{(-i)}$\\
    $\Delta \theta_o = \sqrt{\frac{1}{m} \underset{i}{\sum} (\hat\theta_o - \theta^{(-i)})^2  }  $
\ENDFOR
\STATE \textbf{return} $\{ \Delta \theta\}_o$
\end{algorithmic}
\end{algorithm}

\begin{algorithm}[t]
\caption{Shared-Parameter Scaling Law (ours)}
\label{alg:rho_fit}
\begin{algorithmic}[1]
\REQUIRE Runs $\mathcal{R}$ with optimizers $\mathcal{O}$; each $r\in\mathcal{R}$ has $(o,N,D,L)$ with $o\in\mathcal{O}$.\\
 Scaling law $\hat L(N,D;\theta)= A (\rho_N \cdot N)^{-\alpha} + B (\rho_D \cdot D)^{-\beta} + E$, with parameters $\theta=(A,\alpha,B,\beta,E)$ for the reference optimizer and $\theta_o = \{\rho_N, \rho_D\}_o$ for other optimizers. Nonlinear least-squares solver \textsc{Fit}$(\cdot)$.

\ENSURE Fitted parameters $\{\hat\theta_o\}$.

\FOR{the reference optimizer $o = o_\ast$}
  \STATE Prepare the data\\ $X_o \leftarrow \{(\log N(r), \log D(r)) :  o(r) = o\}$,\\
         $y_o \leftarrow \{\log L(r) : o(r) = o\}$
  \STATE Choose init. $\theta_{init} = (A, \alpha, B, \beta, E)_{init}$ and bounds \{$\theta_{bound}$\}
  \STATE $\hat\theta_\ast \leftarrow \textsc{Fit}\!\left(
      \underset{\theta}{\min} \underset{i}{\sum} \big(y_i - L(X_i;\theta)\big)^2,
      \theta_{init}, \theta_{b}   \right)$
\ENDFOR

\FOR{other optimizers $o \in \mathcal{O} \setminus \{o_\ast \}$}
  \STATE Prepare the data\\ $X_o \leftarrow \{(\log N(r), \log D(r)) :  o(r) = o\}$,\\
         $y_o \leftarrow \{\log L(r) : o(r) = o\}$
  \STATE Choose init. $\theta_{init}= (\rho_N, \rho_D)_{init}$ and bounds \{$\theta_{b}$\}
  \STATE $\hat\theta_o \leftarrow \textsc{Fit}\!\left(
      \underset{\theta}{\min} \underset{i}{\sum} \big(y_i - L(X_i;\theta_\ast, \theta)\big)^2,
      \theta_{init}, \theta_{b}   \right)$
\ENDFOR
\STATE \textbf{return} $\{\hat\theta_o\}$ 
\end{algorithmic}
\end{algorithm}

\section{Method: A Shared-Parameter Scaling Law}
\label{sec:our-method}
\subsection{Unified Scaling Law and Fitting Procedure} 

Motivated by observations in Sec.~\ref{sec:naive-fitting}, we propose a more stable law parametrization that \textit{treats the scaling exponents as shared}, while allowing the rescaling factors to diverge in optimizer-specific fashion. Concretely, we introduce optimizer-specific multiplicative factors $\rho_N$, $\rho_D$ that are interpretable as parameter-efficiency and data-efficiency factors relative to a reference optimizer, chosen naturally as AdamW. 
This leads to a unified optimizer scaling law: 
\begin{equation}
        L = \frac{A}{(N \cdot \textcolor{blue}{\rho_N})^\alpha} + \frac{B}{(D \cdot \textcolor{blue}{\rho_D}) ^\beta}  + E,
        \label{eq:law-with-rhos}
\end{equation}
where the parameters $A, \alpha,B,\beta, E$ are \textit{shared} across all optimizers for a given training instance, but the optimizer scaling parameters $\rho_N, \rho_D$ are \textit{unique} to each optimizer. By fixing the shared exponents, we avoid the parameter co-dependency issue ($A,\alpha$) and ($B, \beta$) identified in Section~\ref{sec:naive-fitting}, reducing the degrees of freedom per optimizer from 5 to just 2. Kumar et al. (\citeyear{kumar2025scaling}) proposed a law with a similar multiplicative factor $\rho_N$ that corresponds to performance degradation under low-precision training.

Here, $\rho_N$ indicates the optimizer behaves as if training a larger 
model (parameter efficiency), while $\rho_D$ indicates faster convergence 
per token (data efficiency). Our formulation allows a natural Pareto front interpretation: if optimizer $O_1$ has both higher data and parameter efficiencies relative to optimizer $O_2$, then it is clearly superior. However, it also allows for trade-offs, where $O_1$ has higher $\rho_N$ but lower $\rho_D$ relative to $O_2$, defying straightforward comparison.

This formulation is not unique without a reference point: one cannot fit $(A, \alpha, \rho_N)$ and $(B, \beta, \rho_D)$ at the same time due to their dependency. There are a few ways to select a reference point, and we choose to fit shared parameters $(A, \alpha, B, \beta, E)$ for the classic baseline optimizer, which is Adam/AdamW. Holding these parameters fixed, we then fit only $(\rho_N, \rho_D)$ per optimizer and obtain parameter and data efficiencies \textit{relative to AdamW}. This procedure is described in Algorithm~\ref{alg:rho_fit}. This allows us to mitigate both the problem of dependent parameters $(A, \alpha)$, $(B, \beta)$ and the ``overfitting'' problem, as described in Section~\ref{sec:naive-fitting}.

In Section~\ref{sec:theory}, we provide a theoretical analysis of gradient descent for the proxy task of a convex quadratic, decomposing the excess loss into approximation and optimization terms akin to \eqref{eq:law-with-rhos}. We prove that both terms scale as power laws with exponents determined by the spectral properties of the underlying problem.

\subsection{Experimental Validation}

In Table~\ref{tab:olmo-rhos}, we first report the fit coefficients, obtained using Algorithm~\ref{alg:rho_fit}, for OLMo family models. 
We repeat the same set of experiments on LLaMa model family. We report these results in Table~\ref{tab:llama-rhos}; both show  similar values and trends, and a good quality-of-fit.  

\begin{table}[h]
    \centering
        \caption{Fitted coefficients $\rho_N,\rho_D$  per optimizer for OLMo-2 family models.}
    \label{tab:olmo-rhos}

    \begin{tabular}{c c c c }
    \toprule
    Optimizer & $\rho_N$ & $\rho_D$ & Fit error\\
\midrule
AdamW & $1.00$ & $1.00$ & $1.53 \cdot 10^{-5}$ \\
Muon & $0.96_{\pm 0.01}$ & $2.08_{\pm 0.08}$ & $3.13 \cdot 10^{-4}$ \\
Scion & $0.95_{\pm 0.01}$ & $1.99_{\pm 0.13}$ & $3.00 \cdot 10^{-4}$ \\
Shampoo & $0.95_{\pm 0.06}$ & $1.55_{\pm 0.48}$ & $3.06 \cdot 10^{-4}$ \\
SOAP & $0.95_{\pm 0.02}$ & $2.57_{\pm 0.25}$ & $2.80 \cdot 10^{-4}$ \\

\bottomrule
    \end{tabular}
\end{table}

\begin{table}[h]
    \centering
    \caption{Fitted coefficients $\rho_N,\rho_D$  per optimizer for LLaMa family models.  }
\label{tab:llama-rhos}
\begin{tabular}{lccc}
\toprule
Optimizer & $\rho_N$ & $\rho_D$ & Fit error \\
\midrule
AdamW & $1.00$ & $1.00$ & $1.53 \cdot 10^{-5}$ \\
Muon & $1.02_{\pm 0.01}$ & $1.41_{\pm 0.08}$ & $7.11 \cdot 10^{-5}$ \\
Scion & $0.95_{\pm 0.01}$ & $1.73_{\pm 0.09}$ & $3.92 \cdot 10^{-5}$ \\
Shampoo & $0.99_{\pm 0.07}$ & $1.16_{\pm 0.13}$ & $7.16 \cdot 10^{-5}$ \\
SOAP & $0.98_{\pm 0.01}$ & $1.75_{\pm 0.13}$ & $5.60 \cdot 10^{-5}$ \\
\bottomrule
\end{tabular}
\end{table}

\paragraph{Discussion.} The coefficients $(\rho_N, \rho_D)$ provide an interpretation of optimizer effects in the scaling laws framework. By construction, $\rho_N$ acts as a parameter-efficiency factor: at a fixed token budget $D$, an optimizer with $\rho_N > 1$ behaves as if ``boosting'' the size of the model, as we are effectively training a model of size $\rho_N \cdot N$. Similarly, $\rho_D$ acts as \textit{data efficiency}: for a given data size $D$ an optimizer with $\rho_D > 1$ will achieve better performance than AdamW. This interpretation makes optimizers comparisons transparent.

If an optimizer has both $\rho_N$ and $\rho_D$ larger than those of a reference, it is strictly better; it achieves better loss along both the model size and data axis. More commonly, optimizers exhibit a trade-off, as can be seen in majority of our experiments, an optimizer increases data efficiency but at a cost of parameter efficiency. In such cases, there is no single ``best'' optimizer because the choice depends on whether a training run is constrained primarily by token budget or model size/memory. 

As $\rho_N$ values stay consistent around $1$ across optimizers and architectures, in both experiments, it appears that the optimizer choice impacts primarily the data efficiency factor $\rho_D$. We emphasize that $\rho_N$ and $\rho_D$ should be viewed only within the fitted $(N, D)$ range, rather than claiming a universal property of the optimizer. 
Yet, our results seem to consistently suggest that modern optimizers are on par or even slightly less parameter-efficient than AdamW, but amply compensate from the point of view of \textit{better data efficiency}. Intuitively, this correlates well with their use of approximate second-order information, which improves their convergence relative to the number of samples.

\begin{figure}
    \centering
    \includegraphics[width=0.95\linewidth]{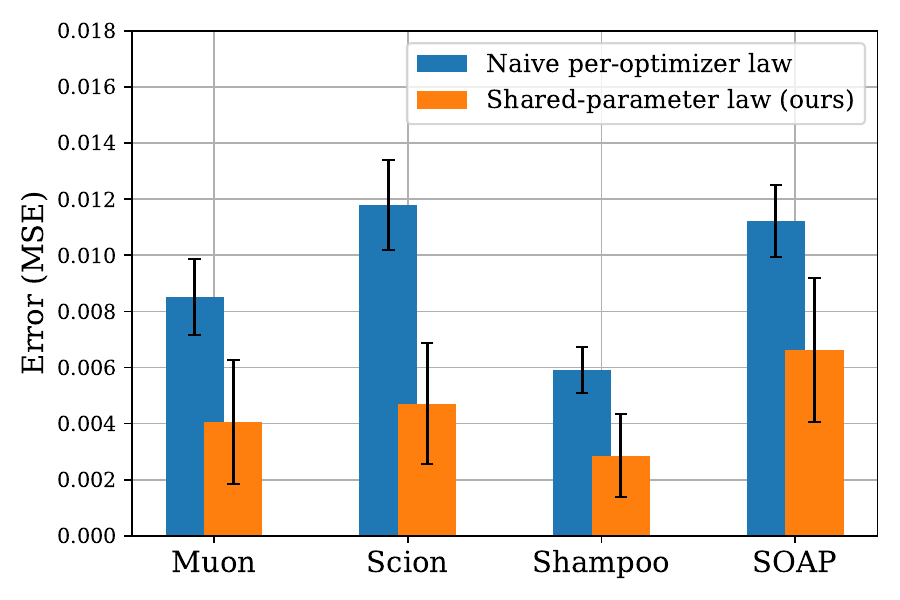}
    \caption{Prediction error (MSE; lower is better) across optimizers for two scaling-law parameterizations: independent per-optimizer Chinchilla fits (“Naive per-optimizer law”) versus our shared-exponent law with optimizer-specific rescaling factors (“Shared-parameter law”). Our approach more than halves the error.}
    \label{fig:pred-power-errors}
\end{figure}

\paragraph{Extrapolation Error Analysis.} To evaluate prediction power in the way scaling laws are typically used in practice, we perform an extrapolation test: we fit each scaling-law variant using only runs with model sizes below 1B parameters and then predict the losses of a larger 1.5B model. The resulting loss errors for the naive approach (Sec.~\ref{sec:naive-fitting}) and the proposed method are reported in Figure~\ref{fig:pred-power-errors}. Across all optimizers, our proposed shared-parameter law improves substantially (by more than 2x) in terms of extrapolation error, confirming its robustness. 

\subsection{Computational Form of the Law}

Optimizers also differ substantially in terms of computational cost: their update rules imply different memory usage and per-step overheads. In particular, preconditioning-based methods such as Shampoo and SOAP require maintaining and applying structured second-moment preconditioners, which typically increases memory use and makes each training step slower than simpler first-order methods such as AdamW or Muon/Scion. For this reason, practitioners are often more interested in a scaling law that uses total compute instead of the amount of training data. We therefore consider an alternative parameterization where the data axis $D$ is replaced by the computational cost $C$:
\begin{equation}
    L = \frac{A}{(N \cdot \rho_N)^\alpha} + \frac{B}{(C \cdot \rho_C)^\beta} + E.
    \label{eq:compute-law}
\end{equation}

Depending on the setting, compute $C$ can be measured as FLOPs, GPU-hours, wall-clock time, energy in kWh, etc. The ($\rho_N, \rho_C$) parameters have a similar interpretation as in the previous subsection: $\rho_N$ and $\rho_C$ capture the trade-off between model size and compute relative to a reference optimizer. 
In our experiments, we measure compute as wall-clock runtime in seconds of each training run on an 8$\times$H100 node. We then fit the same shared-parameter law form, and obtain optimizer-specific efficiency factors $(\rho_N, \rho_C)$, which we report in Table~\ref{tab:compute-rhos}.   
\begin{table}[t]
    \centering
    \caption{Fitted coefficients $\rho_N,\rho_C$  per optimizer for OLMo family models.  }
\label{tab:compute-rhos}
\begin{tabular}{lccc}
\toprule
Optimizer & $\rho_N$ & $\rho_C$ & Fit error \\
\midrule
AdamW & $1.00$ & $1.00$ & $1.01 \cdot 10^{-5}$ \\
Muon & $1.22$ & $1.01$ & $1.97 \cdot 10^{-5}$ \\
Scion & $1.04$ & $1.26$ & $7.89 \cdot 10^{-6}$ \\
Shampoo & $0.97$ & $1.11$ & $2.85 \cdot 10^{-5}$ \\
SOAP & $0.92$ & $1.44$ & $2.26 \cdot 10^{-5}$ \\
\bottomrule
\end{tabular}
\end{table}

We note that the resulting  $(\rho_N, \rho_C)$ factors are heavily dependent on implementation details (hardware utilization, communication, etc.) and may change with different hardware or pipeline optimizations. Nevertheless, we can observe that, in our setup, the Scion optimizer strictly dominates Shampoo, as it has both higher parameter efficiency and higher computational efficiency (due to one-sided preconditioning, which reduces its runtime significantly). 

\section{Theoretical Analysis}
%========= Theory goes here ======

% \subsection{Theoretical Justification}
\label{sec:theory}
In order to provide justification for the proposed form of the law (Eq.~\ref{eq:law-with-rhos}) and, in general, for the Chinchilla law form (Eq.~\ref{eq:classic-law}) itself, we build on top of the previous work \cite{bach2025ztransform, kunstnerbach2025scaling}. There, the authors have provided the exact exponents for infinitely-dimensional quadratic objectives with general spectrum for gradient descent and accelerated gradient descent, as well as extensions to heavy-ball momentum and stochastic gradient descent with a power-law spectrum. 

{\bf Problem Setup.} Following \citet{bach2025ztransform}, we consider minimization of the quadratic loss $L(\theta) = L^* + \frac{1}{2}\langle \theta-\theta^*, H(\theta-\theta^*) \rangle$ for some compact positive semidefinite operator $H : {\cal H} \to {\cal H}$ on a separable Hilbert space of infinite dimension. Applying Gradient Descent (GD) with some step-size $\gamma>0$ to the quadratic objective leads iterations
\begin{align*}
\theta_k
&= \theta_{k-1} - \gamma\nabla L(\theta_{k-1})
= \theta_{k-1} - \gamma H(\theta_{k-1}-\theta^*) \\
&= \theta^* + (I-\gamma H)(\theta_{k-1} - \theta^*)
= \theta^* + (I-\gamma H)^k\delta,
\end{align*}
where $\delta = \theta_{0}-\theta^*$. From this we derive the excess loss as:
\begin{align}
&L(\theta_k) - L^*
= \frac{1}{2\gamma}\sum_{i=1}^{\infty}\lambda_i\langle u_i, \theta_k-\theta^*\rangle^2 \label{loss-sum-repr}\\
% &= \frac{1}{2\gamma}\langle \delta, \gamma H(I-\gamma H)^{2k} \delta \rangle \notag \\
&= \sum_{i=1}^{\infty} \lambda_i (1-\lambda_i)^{2k} \frac{\langle\delta,u_i\rangle^2}{2\gamma}
= \int_0^1 (1-\lambda)^{2k}d\sigma(\lambda) \label{loss-int-repr}
\end{align}
using the spectral decomposition of $\gamma H$ with eigenvalues $\lambda_i\in[0,1]$ and eigenvectors $u_i\in{\cal H}$ for all $i\ge1$, where
\begin{equation}\label{spectral-measure}
    d\sigma(\lambda) = \frac{1}{2\gamma}\sum_{i=1}^{\infty} \lambda_i\langle\delta,u_i\rangle^2 {\rm Dirac}(\lambda_i)
\end{equation}
is the associated weighted spectral measure \citep{berthier2019} with ${\rm Dirac}(\lambda_i)$ being the Dirac measure concentrated at $\lambda=\lambda_i$. Without loss of generality, we assume that $\lambda_i$'s are sorted in non-increasing order. Due to space limitations, we defer the extended theoretical analysis to Appendix \ref{apx:theory}.

While we consider a restrictive convex setting, this approximation is supported by empirical studies of loss landscapes. In particular, \citet{pyhessian, hessianperspective} show that beyond early training, the Hessian spectrum is dominated by positive eigenvalues, and the remaining negative eigenvalues have small absolute values.

{\bf Scaling Laws via Spectral Truncations.} To analyze the scaling of the loss with respect to model dimensionality, we model a network of finite width $d$ as the restriction of the parameter space ${\cal H}$ to the first $d$ eigen-directions.
\begin{definition}[Spectral truncation]%[Width-$d$ Model]
The optimization is constrained to the subspace $\mathcal{H}_d = \{ \theta \in \mathcal{H} \mid \langle\theta, u_i\rangle = 0 \text{ for } i > d \}$ of the first $d$ eigenvectors of ${\cal H}$. The dynamics follow an optimization algorithm on the active subspace, initialized at $\theta_0 = 0$ (which belongs to ${\cal H}_d$ for any $d$).
\end{definition}

With such truncation, the excess loss can be decomposed as:
\begin{align}
& L(\theta_k^d) - L^*
= \frac{1}{2\gamma}\sum_{i=1}^{\infty}\lambda_i\langle u_i, \theta_k^d-\theta^*\rangle^2 \notag \\
&= \underbrace{\frac{1}{2\gamma}\sum_{i>d}\lambda_i\langle u_i,\theta^*\rangle^2}_{{\rm approximation\;error}} + \underbrace{\int_0^1 (1-\lambda)^{2k}d\sigma_d(\lambda)}_{\rm optimization\;error}, \label{loss-decomp}
\end{align}
where $\sigma_d$ is the truncated measure of $\sigma$ \eqref{spectral-measure} composed of only the first $d$ terms in the sum. The second term in \eqref{loss-decomp} is the optimization error which depends both on the number of steps $k$ and dimension $d$. Meanwhile, the first term is the approximation error independent from $k$.

We also need some notion of spectral density to derive asymptotic power laws from the loss decomposition \eqref{loss-decomp}.

\begin{definition}[Spectral dimension]\label{def:spectral-measure}
    We say the spectrum $(\lambda_i)_{i\ge1}$ has finite spectral dimension $\omega>0$ if the associated measure satisfies $\sigma((0,u)) \sim \frac{c}{\omega}u^{\omega}$ as $u\to0^+$ for $c>0$.
\end{definition}

This notion of spectral dimension \cite{berthier2019} is weaker than the one used in \cite{bach2025ztransform} for the exact asymptotic analysis. Thus, examples considered in \cite{bach2025ztransform} applies with the same parametrization $c$ and $\omega$. Particularly, the central example is the one with power decays:
\begin{equation}\label{source-capacity-cond}
    \lambda_i = \frac{\gamma L}{i^{\alpha}}, \quad |\langle\delta, u_i \rangle| = \frac{\Delta}{i^{\nicefrac{\beta}{2}}}, \quad \text{for } i\ge1,
\end{equation}
where $L$ and $\Delta$ are positive constants, $\alpha$ and $\beta$ are positive exponents such that $\alpha+\beta>1$ to ensure the loss $L(0)$ is finite at initialization $\theta_0=0$. In this case, $\omega = 1+\frac{\beta-1}{\alpha}$ and $c = \frac{L\Delta^2}{2\alpha(\gamma L)^{\omega}}$ (see Lemma 3 in  \cite{bach2025ztransform}).

Below we present our main theoretical claim, stated for general spectrum and describing the scaling laws for GD.

\begin{theorem}[Theoretical scaling law]\label{thm:sc-law-finite-dim}
    If the measure \eqref{spectral-measure} associated with eigenvalues $(\lambda_i)_{i\ge1}$ admits spectral dimension $\omega>0$, then we have two phases for the loss scaling\footnote{Notation $a_k = \Theta(b_k)$ means the ratio $a_k/b_k$ converges to some positive finite constant as $k\to\infty$.}.
    
    \underline{Phase 1.} If $k\lambda_d\le1$, then the loss scales as a power law:
    \begin{equation}\label{main-sc-law-power}
    L(\theta_k^d) = L^* + \Theta(\lambda_{d}^{\omega}) + \widetilde{\cal O}(k^{-\omega}).
    \end{equation}
    \underline{Phase 2.} Otherwise, if $k\lambda_d>1$, the power law eventually saturates into exponential rate for the optimization term:
    \begin{equation*}\label{main-sc-law-exp1}
    L(\theta_k^d) = L^* + \Theta(\lambda_{d}^{\omega}) + {\cal O}(e^{-2k\lambda_d}).
    \end{equation*}
\end{theorem}

The obtained theoretical scaling laws \eqref{main-sc-law-power} clearly resemble the empirical scaling law \eqref{eq:classic-law}. The left hand side $L(\theta_k^d)$ is the loss achieved with a model with $d$ active parameters after $k$ optimization steps. The first term $L^*$ on the right hand side corresponds to the irreducible error $E$ in the empirical law. The approximation error $\Theta(\lambda_d^{\omega})$ corresponds to the parameter-efficiency term $\frac{A}{N^{\alpha}}$. Note that the scaling of this term largely depends on the spectrum of the problem and does not have to scale polynomially in general. The exact scaling of this term is $\frac{c}{\omega}\lambda_d^{\omega}$ (see Definition \ref{def:spectral-measure}). Finally, the optimization term $\widetilde{\cal O}(k^{-\omega})$ or ${\cal O}(e^{-2k\lambda_d})$ corresponds to the data-efficiency term $\frac{B}{D^{\beta}}$.
The latter phase of the law is the late-stage of the optimization where the number of gradient descent steps hits the spectral barrier of the problem and the loss converges exponentially fast governed by the smallest active eigenvalue $\lambda_d$ of the quadratic objective.

As a corollary, note that the infinite-dimensional result can be stated by letting $d\to\infty$. Then the term $\Theta(\lambda_d^{\omega})$ vanishes and the saturation phase $k\lambda_d>1$ is never achieved, leaving us $L^* + \widetilde{\cal O}(k^{-\omega})$ for the loss. Up to constants and $\log$-factors, our asymptotic bound \eqref{main-sc-law-power} matches the exact asymptotics of Proposition 1 by \cite{bach2025ztransform}. Furthermore, the bound can be made non-asymptotic by revealing all hidden terms in $\widetilde{\cal O}$ under the same weaker notion of spectral dimension.

{\bf Exact Asymptotics with Power Decays.} In order to derive exact asymptotics for the scaling law, specifically for the optimization error, we restrict the spectrum and the target signal to obey the standard polynomial decays \eqref{source-capacity-cond}. With this spectral conditions we have the following two phases:

\underline{\em Phase 1.} If $k\ll d^{\alpha}$, then the loss scales as a power law:
\begin{equation}\label{asym-sc-law-power}
L(\theta_k^d) \sim L^* + \frac{C_1}{d^{\alpha+\beta-1}} + \frac{C_2}{k^{\omega}}.
\end{equation}
\underline{\em Phase 2.} If $k\gg d^{\alpha}$, then the loss scales as follows:
\begin{equation*}\label{main-sc-law-exp}
L(\theta_k^d) \sim L^* + \frac{C_1}{d^{\alpha+\beta-1}} + \frac{C_3}{d^{\beta-1} k} \exp\left(-2k\frac{\gamma L}{d^{\alpha}}\right),
\end{equation*}
where $\omega = 1 + \frac{\beta-1}{\alpha}$ and constants are defined below:
$$
\textstyle
C_1 = \frac{L\Delta^2}{2(\alpha+\beta-1)},\;
C_2 = \frac{L \Delta^2 \Gamma(\omega)}{2\alpha(2\gamma L)^{\omega}},\;
C_3 = \frac{L \Delta^2}{4\alpha\gamma L}.
$$

{\bf Optimizer-specific Parametrization.} To bridge the empirical law \eqref{eq:law-with-rhos} with the theoretical law \eqref{asym-sc-law-power} (pre-saturation phase), we need to address the optimizer-specific parametrization of the laws. In the theoretical law \eqref{asym-sc-law-power}, the optimizer is reflected as the exponent of the ``data'' term $\frac{C_2}{k^{\omega}}$: the exponent $\omega$ corresponds to the standard gradient descent, and for the (Nesterov) accelerated gradient descent, the exponent becomes $\omega+1$. Meanwhile, in the empirical law \eqref{eq:law-with-rhos}, we have optimizer-specific coefficient $\rho_D$, keeping the exponent $\beta$ fixed relative to some reference optimizer.

To close this modeling gap, we note that empirical and theoretical losses should not be compared directly due to potential scale mismatch. For this reason, we assume that for each term of the law, there exists a non-linear transformation that maps the loss from the theoretical law to the empirical law scale. In particular, if $T_D$ is the transformation for the ``data'' term and the optimization steps $k$ are treated as the number of tokens $D$, then $T_D(\frac{C_2}{D^{\omega}}) = \frac{B}{(D\rho_D)^{\beta}}$.

As we are interested in asymptotic relations with large $D$, we use an approximation $T_D(u)\sim C_Du^{p_D}$ as $u\to0^+$ for some positive exponent $p_D$ and constant $C_D$. Hence,
$
T_D\left( \frac{C_2}{D^{\omega}}\right) \sim C_D\left( \frac{C_2}{D^{\omega}}\right)^{p_D} = \frac{B}{(D \cdot \rho_D)^\beta}.
$
Matching the powers of $D$, we get $p_D = \nicefrac{\beta}{\omega}$, from which we conclude $\rho_D = (B/C_D)^{\nicefrac{1}{\beta}}C_2^{-\nicefrac{1}{\omega}}$. Thus, the optimizer specific exponent $\omega$ in the theoretical law indeed appears in the coefficient $\rho_D$ in the empirical law.

\section{Conclusion}

We asked whether an optimizer choice can be included into empirical scaling laws for LLM pretraining in a way that is both stable and interpretable. We found that fitting separate Chinchilla-form laws per optimizer leads to poorly identified and unreliable parametrization. 
To address this, we introduced a shared-exponent scaling law that keeps the Chinchilla form but models optimizer effects as multiplicative rescaling of model size and data. Concretely, our law is sharing parameters $(A, \alpha, B, \beta, E)$ across optimizers with optimizer-specific model- and data efficiency $\rho_N, \rho_D$ respectively. This parameterization is substantially more stable and yields an interpretable summary of optimizer differences, which can be used directly for compute planning. Further, we also provide theoretical justification for Chinchilla-type scaling laws.

\section*{Acknowledgements}

We thank NVIDIA and Google corporation for their grants, which supported part of this research. Alexandra Volkova was supported in part by the BILAI Cluster of Excellence program. 

Mher Safaryan was supported by Research England under the Expanding Excellence in England (E3) funding stream, which was awarded to MARS: Mathematics for AI in Real-world Systems in the School of Mathematical Sciences at Lancaster University.

We would like to thank our contacts at Datacrunch/Verda, Paul Chang and Antonio Dominguez, for hardware support that was essential to this project. This research was supported by the Scientific Service Units (SSU) of IST Austria through resources provided by Scientific Computing (SciComp). This work was supported under project ID 40 as part of the Swiss AI Initiative, through a grant from the ETH Domain and computational resources provided by the Swiss National Supercomputing Centre (CSCS) under the Alps infrastructure.

\section*{Impact Statement}

This paper presents work whose goal is to advance the field of Machine Learning. There are many potential societal consequences of our work, none of which we feel must be specifically highlighted here.

\bibliography{references}
\bibliographystyle{icml2026}

%%%%%%%%%%%%%%%%%%%%%%%%%%%%%%%%%%%%%%%%%%%%%%%%%%%%%%%%%%%%%%%%%%%%%%%%%%%%%%%
%%%%%%%%%%%%%%%%%%%%%%%%%%%%%%%%%%%%%%%%%%%%%%%%%%%%%%%%%%%%%%%%%%%%%%%%%%%%%%%
% APPENDIX
%%%%%%%%%%%%%%%%%%%%%%%%%%%%%%%%%%%%%%%%%%%%%%%%%%%%%%%%%%%%%%%%%%%%%%%%%%%%%%%
%%%%%%%%%%%%%%%%%%%%%%%%%%%%%%%%%%%%%%%%%%%%%%%%%%%%%%%%%%%%%%%%%%%%%%%%%%%%%%%
\newpage
\appendix
\onecolumn
\section{Model configurations}
\label{apx:model-configs}
Tables \ref{tab:olmo_model_configs} and \ref{tab:llama_model_configs} summarize the OLMo and LLaMa model configurations used in our experiments, respectively. For each nominal size label, we report the exact total parameter count used in the scaling-law fits (including embeddings), along with the number of transformer blocks and attention heads (with head dimension fixed to 128). 

\begin{table}[h]
\centering
\caption{OLMo model configurations used in our experiments. ``Model size (M)'' denotes the nominal model size label, while ``Total params'' reports the exact total parameter count used in our fits (including embeddings).}
\label{tab:olmo_model_configs}
\small
\begin{tabular}{c c c c}
\toprule
\textbf{Model size, M} & \textbf{Total parameter count} & \textbf{\# Blocks} & \textbf{\# Heads} \\
\midrule
  50   &   51 754 764   &  2 &  6 \\
 140   &  143 701 686   &  6 &  9 \\
 250   &  254 519 892   &  7 & 12 \\
 350   &  355 195 524   & 11 & 12 \\
 500   &  525 792 972   & 17 & 12 \\
 760   &  766 454 655   & 17 & 15 \\
1500   & 1 533 199 554  & 25 & 18 \\
\bottomrule
\end{tabular}
\end{table}

\begin{table}[h]
\caption{LLaMA model configurations used in our experiments. ``Model size (M)'' denotes the nominal model size label, while ``Total params'' reports the exact total parameter count used in our fits.}
\label{tab:llama_model_configs}
\centering
\small
\setlength{\tabcolsep}{8pt}
\begin{tabular}{c c c c}
\toprule
\textbf{Model size, M} & \textbf{Total parameter count} & \textbf{\# Blocks} & \textbf{\# Heads} \\
\midrule
  50   &   46 708 736   &  5 &  4 \\
  70   &   71 494 400   &  6 &  5 \\
 180   &  185 698 304   &  8 &  8 \\
 400   &  378 933 760   & 12 & 10 \\
 600   &  605 721 600   & 14 & 12 \\
1400   & 1 445 187 584  & 20 & 16 \\
\bottomrule
\end{tabular}
\end{table}

\section{Scaling Behavior}
Figure~\ref{fig:scaling_olmo} visualizes the empirical scaling behavior across optimizers. We plot validation loss as a function of model size $N$ at fixed token-to-parameter ratios $D/N$, overlaying the fitted scaling curves.

\begin{figure}[b]
    \centering
    \includegraphics[width=0.95\linewidth]{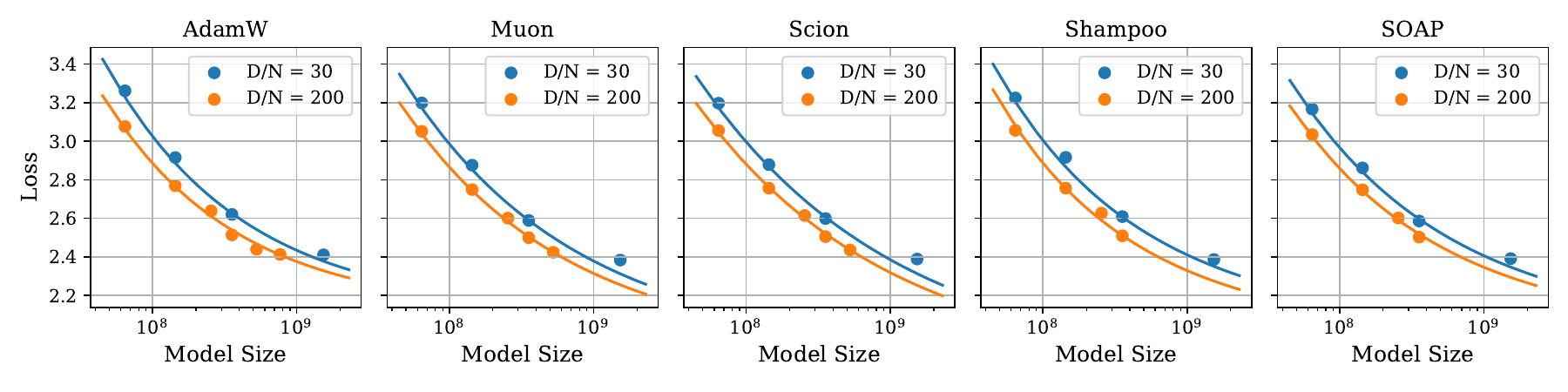}
    \caption{Loss as a function of model size for OLMo family models for token-to-parameter ratios 30 and 200. Data points are measured runs, lines are best-fit scaling curves.}
    \label{fig:scaling_olmo}
\end{figure}

%%%%%%%%%%%%%%%%%%%%%%%%%%%%%%%%%%%%%%%%%%%%%%%%%%%%%%%%%%%%%%%%%%%%%%%%%%%%%%%
%%%%%%%%%%%%%%%%%%%%%%%%%%%%%%%%%%%%%%%%%%%%%%%%%%%%%%%%%%%%%%%%%%%%%%%%%%%%%%%
\section{Deferred Proofs with Extended Theoretical Analysis}\label{apx:theory}

Following \citet{bach2025ztransform}, we consider minimization of the quadratic loss $L(\theta) = L^* + \frac{1}{2}\langle \theta-\theta^*, H(\theta-\theta^*) \rangle$ for some compact positive semidefinite operator $H : {\cal H} \to {\cal H}$ on a separable Hilbert space of infinite dimension. Applying Gradient Descent (GD) with some step-size $\gamma>0$ to the quadratic objective leads iteration of the form
\begin{align*}
\theta_k
&= \theta_{k-1} - \gamma\nabla L(\theta_{k-1}) \\
&= \theta_{k-1} - \gamma H(\theta_{k-1}-\theta^*) \\
&= \theta^* + (I-\gamma H)(\theta_{k-1} - \theta^*) \\
&= \theta^* + (I-\gamma H)^k\delta,
\end{align*}
where $\delta = \theta_{0}-\theta^*$. From this we derive the excess loss as:
\begin{align}
L(\theta_k) - L^*
&= \frac{1}{2\gamma}\sum_{i=1}^{\infty}\lambda_i\langle u_i, \theta_k-\theta^*\rangle^2 \label{apx:loss-sum-repr}\\
% &= \frac{1}{2\gamma}\langle \delta, \gamma H(I-\gamma H)^{2k} \delta \rangle \notag \\
&= \sum_{i=1}^{\infty} \lambda_i (1-\lambda_i)^{2k} \frac{\langle\delta,u_i\rangle^2}{2\gamma}
= \int_0^1 (1-\lambda)^{2k}d\sigma(\lambda) \label{apx:loss-int-repr}
\end{align}
using the spectral decomposition of $\gamma H$ with eigenvalues $\lambda_i\in[0,1]$ and eigenvectors $u_i\in{\cal H}$ for all $i\ge1$, where
\begin{equation}\label{apx:spectral-measure}
    d\sigma(\lambda) = \frac{1}{2\gamma}\sum_{i=1}^{\infty} \lambda_i\langle\delta,u_i\rangle^2 {\rm Dirac}(\lambda_i)
\end{equation}
is the associated weighted spectral measure \citep{berthier2019} with ${\rm Dirac}(\lambda_i)$ being the Dirac measure concentrated at $\lambda=\lambda_i$. Without loss of generality, we assume that $\lambda_i$'s are sorted in non-increasing order.

{\bf Scaling laws in infinite-dimensions.} We first revisit infinite-dimensional quadratic problem and derive power laws with simpler arguments and weaker spectral condition. The analysis developed here provides asymptotic bounds with the same exponents as in \citet{bach2025ztransform} for the exact asymptotics. The purpose is to build intuition with simpler analysis and later extend it to the finite-dimensional setting.

\begin{definition}[Spectral dimension]
    We say the spectrum $(\lambda_i)_{i\ge1}$ has finite spectral dimension $\omega>0$ if the associated measure satisfies $\sigma((0,u)) \sim \frac{c}{\omega}u^{\omega}$ as $u\to0^+$ for $c>0$.
\end{definition}

This notion of spectral dimension \cite{berthier2019} is weaker than the one used in \cite{bach2025ztransform} for the exact asymptotic analysis. Thus, examples considered in \cite{bach2025ztransform} applies with the same parametrization $c$ and $\omega$. Particularly, the central example is the measure with power decays
\begin{equation}\label{apx:source-capacity-cond}
    \lambda_i = \frac{\gamma L}{i^{\alpha}}, \quad |\langle\delta, u_i \rangle| = \frac{\Delta}{i^{\nicefrac{\beta}{2}}}, \quad \text{for } i\ge1,
\end{equation}
where $L$ and $\Delta$ are positive constants, $\alpha$ and $\beta$ are positive exponents such that $\alpha+\beta>1$ to ensure the loss $L(0)$ is finite at initialization $\theta_0=0$. In this case, $\omega = 1+\frac{\beta-1}{\alpha}$ and $c = \frac{L\Delta^2}{2\alpha(\gamma L)^{\omega}}$ (see Lemma 3 in  \cite{bach2025ztransform}).

\begin{theorem}\label{apx:thm:sc-law-inf-dim}
    If $\sigma$ admits spectral dimension $\omega>0$, then
    \begin{equation}\label{apx:sc-law-inf-dim}
    L(\theta_k) = L^* + \widetilde{\cal O}(k^{-\omega}), \quad\text{as}\quad k\to\infty,
    \end{equation}
    where $\widetilde{\cal O}$ hides constants and $\log$-factors.
\end{theorem}
\begin{proof}
    Given the integral representation \eqref{apx:loss-int-repr} of the excess loss, for any $u\in(0,1)$ we split the integral into parts:
    \begin{align*}
        L(\theta_k) - L^*
        &= \int_0^u (1-\lambda)^{2k}d\sigma(\lambda)
         + \int_u^1 (1-\lambda)^{2k}d\sigma(\lambda).
    \end{align*}
    For the first integral, we apply the inequality $(1-\lambda)^{2k}\le1$ and the spectral condition of $\sigma(\lambda)$, to show that is of order ${\cal O}(u^{\omega})$. For the second integral, we apply the inequality $(1-\lambda)^{2k}\le\exp(-2ku)$ for any $\lambda\in(u,1)$ and boundedness of the spectral measure (i.e., $\sigma([0,1])$ is finite), to show that is of order ${\cal O}(\exp(-2ku))$. Therefore, we have
    \begin{equation*}
        L(\theta_k) = L^* + {\cal O}(u^{\omega} + \exp(-2ku)),
    \end{equation*}
    for $u\in(0,1)$ such that $u\to0$ as $k\to\infty$. We are left to choose $u = \frac{\omega\log k}{2k}$ and note that the excess loss is of order
    \begin{align*}
        u^{\omega} + \exp(-2ku)
        &= \left(\frac{\omega\log k}{2k}\right)^{\omega} + \exp(-\omega\log k) \\
        &= \left(\frac{\omega\log k}{2k}\right)^{\omega} + \frac{1}{k^{\omega}}
        = \widetilde{\cal O}\left(\frac{1}{k^{\omega}}\right),
    \end{align*}
    which completes the proof.
\end{proof}
Up to constants and $\log$-factors, our asymptotic bound \eqref{apx:sc-law-inf-dim} matches the exact asymptotics of Proposition 1 by \cite{bach2025ztransform}. Furthermore, the bound can be made non-asymptotic by revealing all hidden terms in $\widetilde{\cal O}$ under the same weaker notion of spectral dimension.

{\bf Scaling laws via spectral truncations.} We now extend the previous analysis to finite-dimensional setting and identity the scaling with respect to dimension. We model a network of finite width $d$ as the restriction of the parameter space to the first $d$ eigen-directions.
\begin{definition}[Width-$d$ Model]
The optimization is constrained to the subspace $\mathcal{H}_d = \{ \theta \in \mathcal{H} \mid \langle\theta, u_i\rangle = 0 \text{ for } i > d \}$ spanning the first $d$ eigenvectors of ${\cal H}$. The dynamics follow an optimization algorithm on the active subspace, initialized at $\theta_0 = 0$ (which belongs to ${\cal H}_d$ for any $d$).
\end{definition}

With such truncation, the excess loss can be decomposed as:
\begin{align}
L(\theta_k^d) - L^*
&= \frac{1}{2\gamma}\sum_{i=1}^{\infty}\lambda_i\langle u_i, \theta_k^d-\theta^*\rangle^2 \notag \\
&= \frac{1}{2\gamma}\sum_{i>d}\lambda_i\langle u_i,\theta^*\rangle^2 + \frac{1}{2\gamma}\sum_{i=1}^{d}\lambda_i\langle u_i, \theta_k^d-\theta^*\rangle^2 \notag \\
&= \frac{1}{2\gamma}\sum_{i>d}\lambda_i\langle u_i,\theta^*\rangle^2 + \frac{1}{2\gamma}\sum_{i=1}^{d} \lambda_i (1-\lambda_i)^{2k} \langle\delta,u_i\rangle^2 \notag \\
&= \underbrace{\frac{1}{2\gamma}\sum_{i>d}\lambda_i\langle u_i,\theta^*\rangle^2}_{{\cal E}_{\rm approx}(d)} + \underbrace{\int_0^1 (1-\lambda)^{2k}d\sigma_d(\lambda)}_{{\cal E}_{\rm opt}(k,d)}, \label{apx:loss-decomp}
\end{align}
where $\sigma_d$ is the truncated measure of $\sigma$ \eqref{apx:spectral-measure} composed of $d$ terms in the sum. The second term ${\cal E}_{\rm opt}(k,d)$ in \eqref{apx:loss-decomp} is the optimization error which depends both on the number of steps $k$ and dimension $d$. Meanwhile, the first term ${\cal E}_{\rm approx}(d)$ is the approximation error independent from $k$.

Using the notion of spectral dimension and the fact that eigenvalues are sorted, we get the following asymptotic expression for the approximation error as $d\to\infty$:
\begin{align*}
    &{\cal E}_{\rm approx}(d)
    = \frac{1}{2\gamma}\sum_{\lambda_i<\lambda_d}\lambda_i\langle u_i,\theta^*\rangle^2
    = \sigma((0,\lambda_d)) \propto \lambda_d^{\omega}.
\end{align*}

For the optimization error ${\cal E}_{\rm opt}(k,d)$, note that it differs from the integral representation \eqref{apx:loss-int-repr} only by the measure $\sigma_d(\lambda)$ which is the truncation version. Because of the truncation, there is no active eigenvalue in the interval $(0,\lambda_d)$ implying $\sigma_d((0,\lambda_d)) = 0$. Hence, in the integral representation of ${\cal E}_{\rm opt}(k,d)$, the support $[0,1]$ can be replaced by $[\lambda_d, 1]$. As we analyzed in the infinite-dimensional case, we split the integral with some intermediate value $u\in(\lambda_d, 1)$:
\begin{align}
    &{\cal E}_{\rm opt}(k,d)
    = \int_{\lambda_d}^1 (1-\lambda)^{2k}d\sigma_d(\lambda) \notag \\
    &= \int_{\lambda_d}^u (1-\lambda)^{2k}d\sigma(\lambda) + \int_{u}^1 (1-\lambda)^{2k}d\sigma(\lambda) \notag \\
    &= {\cal O}\left( e^{-2k\lambda_d}\sigma((\lambda_d,u)) + e^{-2ku} \right). \label{apx:opt-error}
    % &= {\cal O}\left( e^{-2k\lambda_d}(u^{\omega} - \lambda_d^{\omega}) + e^{-2ku} \right) \\
\end{align}
To make the asymptotics of \eqref{apx:opt-error} explicit, we distinguish two phases depending on the relative sizes of $k$ and $d$.

{\bf Phase 1: Under-training or power law ($k\lambda_d\le1$).} This phase mimics the infinite-dimensional setting we considered before. Indeed, using relation $k\lambda_d\le1$, we can further simplify the asymptotic bound \eqref{apx:opt-error} with the same value $u = \frac{\omega\log k}{2k}$ as in the previous analysis:
\begin{align}
    {\cal E}_{\rm opt}(k,d)
    &= {\cal O}\left( e^{-2k\lambda_d}\sigma((\lambda_d,u)) + e^{-2ku} \right) \notag \\
    &= {\cal O}\left( \sigma((\lambda_d,u)) + e^{-2ku} \right) \\
    &= {\cal O}\left( u^{\omega} + e^{-2ku} - \lambda_d^{\omega} \right) \notag \\
    &= {\cal O}\left( \left(\frac{\omega\log k}{2k}\right)^{\omega} + \frac{1}{k^{\omega}} - \lambda_d^{\omega} \right)
    = \widetilde{\cal O}\left(\frac{1}{k^{\omega}} \right) \notag.
\end{align}

{\bf Phase 2: Over-training or saturation ($k\lambda_d>1$).} This is the late-stage of the optimization where the number of gradient descent steps hits the spectral barrier of the problem and the loss converges exponentially fast governed by the smallest positive eigenvalue of the quadratic objective. Indeed, choosing $u = \lambda_d$ (or, equivalently, not splitting the integral representation in the first place) gives us $e^{-2k\lambda_d}$ rate, which makes sense only in this regime $k\lambda_d>1$.

\begin{theorem}\label{apx:thm:sc-law-finite-dim}
    If $\sigma$ admits spectral dimension $\omega>0$, then\footnote{Notation $a_k = \Theta(b_k)$ means the ratio $a_k/b_k$ converges to some positive finite constant as $k\to\infty$.}
    \begin{equation}\label{apx:sc-law-theory-finite}
    L(\theta_k^d) = L^* + \Theta(\lambda_{d}^{\omega}) +
    \begin{cases}
        \widetilde{\cal O}(k^{-\omega}) & \text{if}\quad k\lambda_d\le1, \\
        {\cal O}(e^{-2k\lambda_d}) & \text{if}\quad k\lambda_d>1. \\
    \end{cases}
    \end{equation}
\end{theorem}

The obtained law directly maps to the empirical scaling law \eqref{eq:classic-law}. The left hand side $L(\theta_k^d)$ is the loss achieved with a model with $d$ active parameters after $k$ optimization steps. The first term $L^*$ on the right hand side corresponds to the irreducible error $E$ in the empirical law \eqref{eq:classic-law}. The approximation error $\Theta(\lambda_d^{\omega})$ corresponds to the parameter-efficiency term $\frac{A}{N^{\alpha}}$. Note that the scaling of this term largely depends on the spectrum of the problem and does not have to scale polynomially in general. The exact scaling of this term is $\frac{c}{\omega}\lambda_d^{\omega}$. Finally, the optimization term $\widetilde{\cal O}(k^{-\omega})$ or ${\cal O}(e^{-2k\lambda_d})$ corresponds to the data-efficiency term $\frac{B}{D^{\beta}}$. For the under-training regime, we proved a power law relationship with the same exponent (being the spectral dimension of the problem) as in the infinite-dimensional setting.

As a sanity check, notice that the infinite-dimensional result in Theorem \ref{apx:thm:sc-law-inf-dim} is covered by this: letting $d\to\infty$, the term $\Theta(\lambda_d^{\omega})$ vanishes and the phase $k\lambda_d>1$ is never achieved, leaving us $L^* + \widetilde{\cal O}(k^{-\omega})$ for the loss. Also, note that the obtained asymptotic relation for the approximation error is exact, while for the optimization error it is an upper bound.

{\bf Exact asymptotics: polynomial decays for the spectrum
and target signal.} In order to derive exact asymptotics for the scaling law, we restrict the spectrum and the target signal to obey the standard polynomial decays, also known as ``source'' and ``capacity'' conditions \eqref{apx:source-capacity-cond}:

The analysis of Theorem \ref{thm:sc-law-finite-dim} proves that ${\cal E}_{\rm approx}(d) = \sigma((0,\lambda_d))$ for any spectrum. As discussed earlier, for the spectrum \eqref{apx:source-capacity-cond} with power decay, we have $\sigma((0,u))\sim\frac{c}{\omega}u^{\omega}$ with $\omega = 1+\frac{\beta-1}{\alpha}$ and $c = \frac{L\Delta^2}{2\alpha(\gamma L)^{\omega}}$. Therefore, plugging the obtained expressions and the exact form of $\lambda_d$, we get:
$$
{\cal E}_{\rm approx}(d)
% \sim \frac{c}{\omega}\lambda_d^{\omega}
\sim \frac{L\Delta^2}{2\alpha\omega(\gamma L)^{\omega}}\left(\frac{\gamma L}{d^{\alpha}} \right)^{\omega}
=    \frac{L\Delta^2 d^{-(\alpha+\beta-1)}}{2(\alpha+\beta-1)}.
$$
Next, for the optimization error ${\cal E}_{\rm opt}(k,d)$ we have:
\begin{equation*}
    {\cal E}_{\rm opt}(k,d)
    \sim \frac{L \Delta^2}{2\alpha(2\gamma L)^{\omega}} \frac{\Gamma(\omega, 2\gamma Lkd^{-\alpha})}{k^{\omega}}.
\end{equation*}
\begin{proof}
We start from the definition of the optimization error, apply the asymptotic relation for the exponential function $e^x\sim1+x$ as $x\to0$, then we represent the finite sum as its asymptotically equivalent integral expression since boundary terms vanish in the limit $k\to\infty$:
\begin{align}
    {\cal E}_{\rm opt}(k,d)
    &= \frac{1}{2\gamma}\sum_{i=1}^{d} \lambda_i (1-\lambda_i)^{2k} \langle\delta,u_i\rangle^2 \notag \\
    &= \frac{L\Delta^2}{2}\sum_{i=1}^{d} \frac{1}{i^{\alpha+\beta}} \left(1-\frac{\gamma L}{i^{\alpha}}\right)^{2k} \notag \\
    &\sim \frac{L\Delta^2}{2}\sum_{i=1}^{d} \frac{1}{i^{\alpha+\beta}} \exp\left(-2k\gamma L i^{-\alpha}\right) \notag \\
    &\sim \frac{L\Delta^2}{2}\int_{1}^{d} \frac{1}{x^{\alpha+\beta}} \exp\left(-2k\gamma L x^{-\alpha}\right)dx. \label{apx:opt-error-int-term}
\end{align}
Next, we perform change of variables $t = 2 \gamma L k x^{-\alpha}$. From this we get $x = (2 \gamma L k)^{1/\alpha} t^{-1/\alpha}$ and $dx = -\frac{1}{\alpha} (2 \gamma L k)^{1/\alpha} t^{-1/\alpha - 1} dt$. Furthermore, the limits of integration would become $t_{\min} = 2\gamma Lk d^{-\alpha}$ and $t_{\max} = 2 \gamma L k$. Substituting these expression back into \eqref{apx:opt-error-int-term} (the negative sign of $dx$ swaps the order of limits), we have
\begin{equation*}
\frac{L \Delta^2}{2} \int_{t_{\min}}^{t_{\max}} \left[ (2 \gamma L k)^{-\frac{\alpha+\beta}{\alpha}} t^{\frac{\alpha+\beta}{\alpha}} \right] e^{-t} \left[ \frac{1}{\alpha} (2 \gamma L k)^{\frac{1}{\alpha}} t^{-\frac{1}{\alpha} - 1} \right] dt.
\end{equation*}
Combining the exponents of $(2 \gamma L k)$, $\frac{1}{\alpha} - \frac{\alpha+\beta}{\alpha} = \frac{1-\alpha-\beta}{\alpha} = -(1 + \frac{\beta-1}{\alpha}) = -\omega$, and the exponents of $t$, $\frac{\alpha+\beta}{\alpha} - \frac{1}{\alpha} - 1 = \frac{\beta-1}{\alpha} = \omega - 1$, implies
\begin{equation*}
    {\cal E}_{\rm opt}(k,d)
    \sim \frac{L \Delta^2}{2\alpha} (2 \gamma L k)^{-\omega} \int_{t_{\min}}^{t_{\max}} t^{\omega-1} e^{-t} dt.
\end{equation*}
The integral is exactly $\Gamma(\omega, t_{\min}) - \Gamma(\omega, t_{\max})$, where $\Gamma(\omega, \tau) = \int_{\tau}^\infty t^{\omega-1} e^{-t}dt$ is the upper incomplete Gamma function. As $k\to\infty$, we have $t_{\max} = 2\gamma Lk \to \infty$, and the second term vanishes, leaving $\Gamma(\omega, 2\gamma Lkd^{-\alpha})$.
\begin{equation*}
    {\cal E}_{\rm opt}(k,d)
    \sim \frac{L \Delta^2}{2\alpha(2\gamma L)^{\omega}} \frac{\Gamma(\omega, 2\gamma Lkd^{-\alpha})}{k^{\omega}}.
\end{equation*}
\end{proof}

{\bf Phase 1: Under-training or power law phase ($k\ll d^{\alpha}$).} In this case, lower limit $t_{\min}\to0$ and $\Gamma(\omega, 2\gamma Lkd^{-\alpha})\to\Gamma(\omega)$ leading us to the standard power law:
\begin{equation*}
    {\cal E}_{\rm opt}(k,d)
    \sim \frac{L \Delta^2}{2\alpha(\gamma L)^{\omega}} \frac{\Gamma(\omega)}{(2k)^{\omega}}.
\end{equation*}

{\bf Phase 2: Over-training or saturation ($k\gg d^{\alpha}$).} In this case, $t_{\min}\to\infty$. Using $\Gamma(\omega, \tau) \sim \tau^{\omega-1} e^{-\tau}$ as $\tau\to\infty$, the optimization error decays exponentially:
\begin{equation*}
    {\cal E}_{\rm opt}(k,d)
    \sim \frac{L \Delta^2}{4\alpha\gamma L}\frac{1}{d^{\beta-1} k} \exp\left(-2k\frac{\gamma L}{d^{\alpha}}\right).
\end{equation*}
This confirms that width $d$ imposes a spectral barrier; beyond $k \sim d^\alpha$, convergence is governed by the smallest active curvature $\lambda_d \propto d^{-\alpha}$.

%%%%%%%%%%%%%%%%%%%%%%%%%%%%%%%%%%%%%%%%%%%%%%%%%%%%%%%%%%%%%%%%%%%%%%%%%%%%%%%
%%%%%%%%%%%%%%%%%%%%%%%%%%%%%%%%%%%%%%%%%%%%%%%%%%%%%%%%%%%%%%%%%%%%%%%%%%%%%%%

\end{document}